\documentclass[letterpaper, 10 pt, conference]{ieeeconf}

\usepackage{mathtools}

\usepackage[T1]{fontenc}
\usepackage[latin9]{inputenc}
\usepackage{mathtools}
\usepackage{amsmath}
\usepackage{amsthm}
\usepackage{amssymb}
\usepackage{graphicx}
\usepackage[pagebackref]{hyperref}
\usepackage{array}
\usepackage{balance}
\usepackage{subfigure}
\usepackage{color}

\makeatletter
\newtheorem*{theorem*}{Theorem}

\newtheorem{lem}{Lemma}
\newtheorem*{lem*}{Lemma}
\newtheorem{rem}{Remark}
\newtheorem{ex}{Example}
\IEEEoverridecommandlockouts
\begin{document}

\title{Design of Admissible Heuristics for Kinodynamic Motion Planning via Sum-of-Squares Programming}

\author{Brian Paden$^*$, Valerio Varricchio$^*$, and Emilio Frazzoli$^*$
\thanks{$^*$The authors are with the Laboratory for Information and Decision Systems at MIT (e-mail: bapaden@mit.edu; valerio@mit.edu; frazzoli@mit.edu).}}
\maketitle
\begin{abstract}
%
How does one obtain an admissible heuristic for a kinodynamic motion planning problem?
This paper develops the analytical tools and techniques to answer this question.
A sufficient condition for the admissibility of a heuristic is presented which can be checked directly from the problem data. 
This condition is also used to formulate a concave program to optimize an admissible heuristic.
This optimization is then approximated and solved in polynomial time using sum-of-squares programming techniques.
A number of examples are provided to demonstrate these concepts.
\end{abstract}
\section{Introduction}

Many graph search problems arising in robotics and artificial intelligence that would otherwise be intractable can be solved efficiently with an effective heuristic informing the search.
However, efficiently obtaining a shortest path on a graph requires the heuristic to be admissible as described in the seminal paper introducing the $\rm A^*$ algorithm~\cite{hart1968formal}. In short, an admissible heuristic provides an estimate of the optimal cost to reach the goal from every vertex, but never overestimates the optimal cost.

A major application for admissible heuristics is in searching graphs approximating robotic motion planning problems. 
The workhorse heuristic in kinematic shortest path problems is the Euclidean distance from a given state to the goal.
Figure \ref{fig:intro_demo} demonstrates the benefit of using this heuristic on a typical shortest path problem where informing the search reduces the number iterations required to find a solution by $67\%$.
More recently, methods have been developed for generating graphs approximating optimal trajectories in kinodynamic motion planning problems. 
Notable examples include the kinodynamic variant of the $\rm RRT^*$ algorithm~\cite{karaman2010optimal}, the state augmentation technique proposed in~\cite{hauser2015asymptotically}, and the $\rm GLC$ algorithm~\cite{paden2016generalized}.
While this is not a comprehensive literature review on optimal kinodynamic motion planning, the use of admissible heuristics has been proposed for each of these methods (the use of heuristics for $\rm RRT^*$ was proposed recently in~\cite{informedRRT,batchInformedRRT}).  
The kinodynamic motion planning problem and the use of admissible heuristics are reviewed in Section \ref{sec:Prelim} and \ref{sec:admissible_heuristics} respectively.
%

%
A good heuristic is one which closely underestimates the optimal cost-to-go from every vertex. 
This enables a larger number of provably suboptimal paths to be identified and discarded from the search.
While admissibility of a heuristic is an important concept it gives rise to two challenging questions: (i) Without a priori knowledge of the optimal cost-to-go, how do we verify the admissibility of a candidate heuristic? (ii) How do we systematically construct good heuristics for kinodynamic motion planning problems?    

\begin{figure*}\label{fig:intro_demo}
	\centering{}
	\includegraphics[width=1.0\columnwidth]{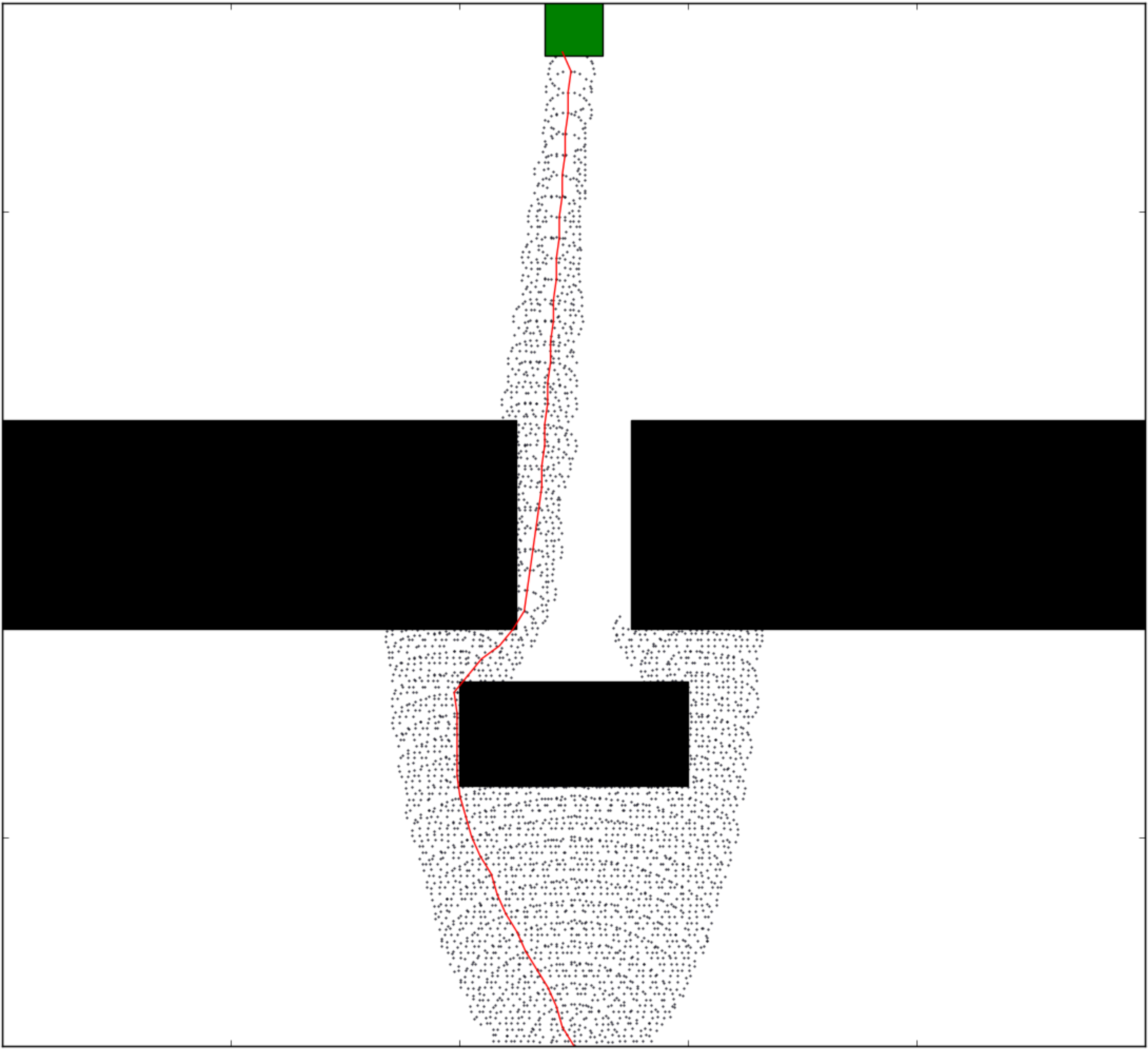}
	\includegraphics[width=1.0\columnwidth]{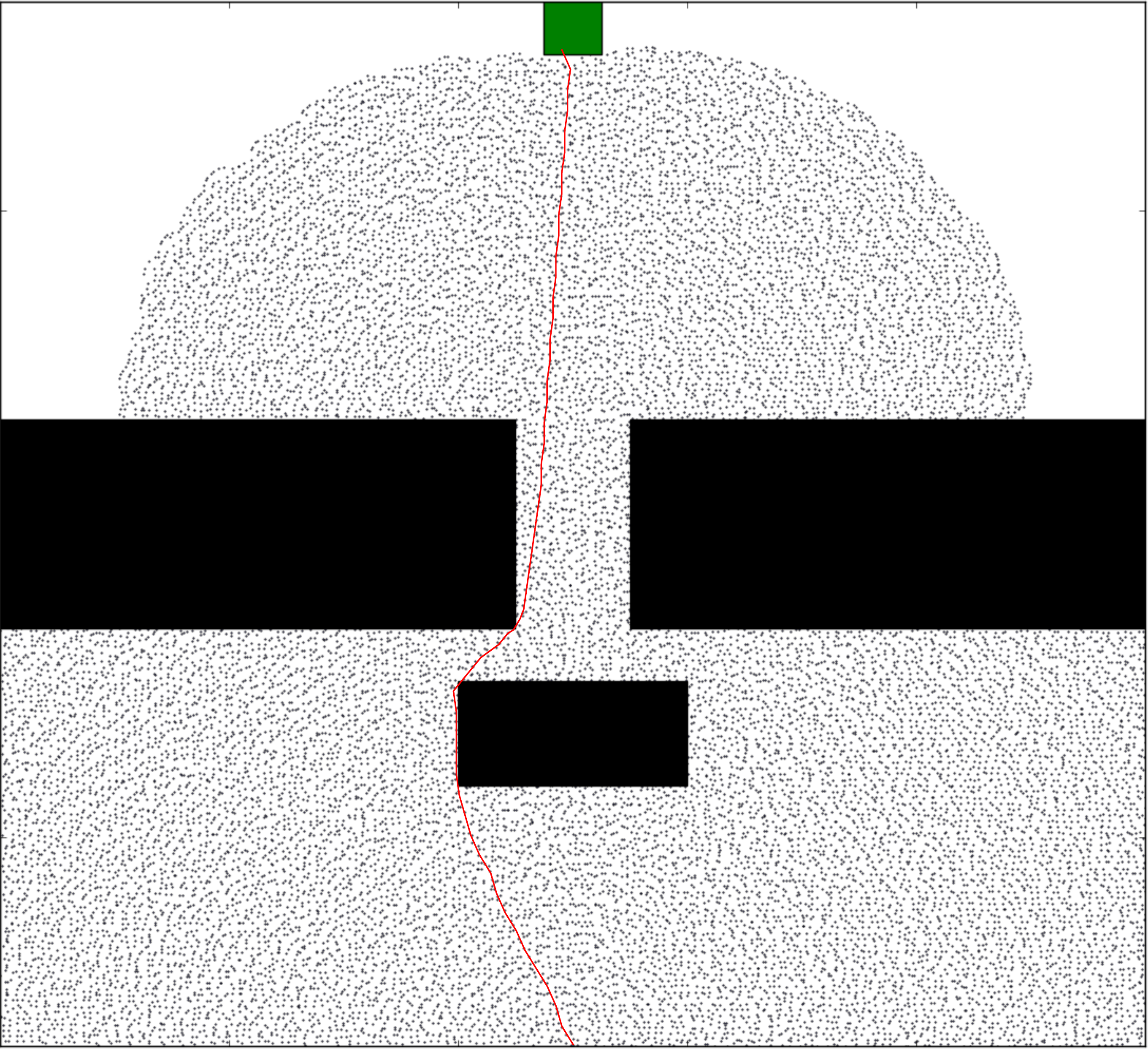}
	\caption{A classic example where an admissible heuristic speeds up a search. Approximate shortest kinematic paths in a 2D environment computed with the generalized label correcting (GLC) method~\cite{paden2016generalized} are shown. Black dots represent vertices of the graph evaluated during the search. The algorithm was executed with (left) and without (right) an admissible heuristic. While the underlying graph is identical, the informed GLC method obtains a solution in 5203 iterations while the standard GLC method obtains a solution in 19030 iterations.}
\end{figure*}
The first question is addressed in Section \ref{sec:admissible_heuristics} where a sufficient condition for the admissibility of a candidate heuristic is presented.
This condition takes the form of an affine inequality involving the heuristic and given problem data.
The result provides a general analytical tool for validating a heuristic constructed by intuition about the problem. 

The second question is addressed in Section \ref{sec:inf_lp}. The admissibility condition is used to formulate a concave maximization over the space of candidate heuristics.
The objective of the optimization is constructed so that the optimal cost-to-go is a globally optimal solution to the optimization.
The approach to analyzing and constructing admissible heuristics is inspired from the dual formulation to the trajectory optimization problem~\cite{fleming1988generalized,vinter1993convex}. 

Section \ref{sec:Characterising-Admissibility-as} outlines a computational procedure for approaching the optimization.
A finite dimensional subspace of polynomials is used to approximate the space of heuristics. 
Sum-of-squares (SOS) programming~\cite{parrilo2004sum} techniques are then used to obtain an approximate solution in polynomial time. 
In doing so we provide the first general procedure to compute admissible heuristics to kinodynamic motion planning problems. 

Examples demonstrating how to use the admissibility condition to verify that a heuristic is admissible as well as numerical examples of the SOS programming approach are provided in Section \ref{sec:Examples_usage}.
The YALMIP~\cite{yalmip} scripts used to compute the example heuristics can be found in~\cite{SOS_heuristics}.
%

%
%
%
%

\section{Kinodynamic Motion Planning\label{sec:Prelim}}

Consider a system whose state at time $t\in \mathbb{R}$ is described by a vector in $\mathbb{R}^n-$the state space.
A \textit{trajectory} $x$ representing a time evolution of the system state is a continuous map from a closed time domain $[0,T]$ to the state space, $x:[0,T]\rightarrow \mathbb{R}^n$ for some $T>0$.
A trajectory in a kinodynamic motion planning problem must satisfy several point-wise constraints. 
First, a subset $X_\mathrm{free}\subset \mathbb{R}^n$ of the state space encodes the set of allowable states over the entire domain of the trajectory; $x(t)\in X_\mathrm{free}$ for all $t\in [0,T]$.
Secondly, there is an initial state constraint, $x(0)=x_0$ for a specified state $x_0\in X_\mathrm{free}$.
Lastly,   there is a terminal constraint;  $x(T) \in X_\mathrm{goal}$ for a specified subset $X_\mathrm{goal}\subset X_\mathrm{free}$. 

In addition to the point-wise constraints, the trajectory must satisfy differential constraints.
At each time $t$ the system is affected by a control action $u(t)$. 
The set of available control actions is a subset $\Omega$ of $\mathbb{R}^m$. 
The time history of control actions is referred to as a \textit{control signal} and unlike a trajectory it need not be continuous. 
However, the control signal is assumed to be Lebesgue integrable and essentially bounded. 
The control action affects the trajectory through the differential equation,
\begin{equation}
\dot{x}(t)=f(x(t),u(t)),\label{eq:ode}
\end{equation}
where $f:\mathbb{R}^n \times \mathbb{R}^m \rightarrow \mathbb{R}^n$.
A trajectory $x$ with domain $[0,T]$ must satisfy (\ref{eq:ode}) for some control signal $u$ for almost all $t\in[0,T]$.
A \textit{feasible} trajectory is one that satisfies these point-wise and differential constraints.
Next, a cost functional $J$ provides a way to quantify the merit of a candidate trajectory and control signal, 
\begin{equation}
J(x,u)=\int_{[0,T]}g(x(t),u(t))\,\, \mu(dt).\label{eq:cost}
\end{equation}
It is assumed $g(z,w)\geq 0$ so that a nonnegative running cost is associated to each state-action pair. The measure $\mu$ is the standard Lebesgue measure.
A solution to an optimal kinodynamic motion planning problem is a feasible trajectory and control signal which minimizes (\ref{eq:cost}).
While the minimum of (\ref{eq:cost}) may not be attained, the optimal value from the initial state $x_0$ is always well defined. 
\subsection{The Value Function}
The cost-to-go or (optimal) value function $V:X_{free}\rightarrow \mathbb{R}$ describes the greatest lower bound on the cost to reach the goal set from the initial state $z\in X_{free}$.
The following properties of $V$ follow immediately from the assumption $g(z,w) \geq 0$ in (\ref{eq:cost}):
\begin{equation}\arraycolsep=1.4pt\def\arraystretch{1.5}
	\begin{array}{c}
		V(z) \geq 0,\qquad \forall z \in X_\mathrm{free}, \\
		V(z) = 0, \qquad \forall z \in X_\mathrm{goal}.
	\end{array}\label{eq:positive_edge_weight}
\end{equation} 
If the value function $V$ is differentiable it is a classical\footnote{The gradient of $V$ is well defined and the equation is satisfied for all $x\in X_\mathrm{free}\setminus \bar{X}_\mathrm{goal}$. In some cases the value function is not differentiable in which case a generalized solution concept known as a viscosity solution is used~\cite{crandall1983viscosity}.} solution to the Hamilton-Jacobi-Bellman (HJB) equation:
\begin{equation}\label{eq:HJB}
\begin{array}{c}
\underset{{w\in \Omega}}{\inf} \left\{ \left\langle \nabla V(z), f(z,w) \right\rangle + g(z,w) \right\}=0, \\\forall z\in X_{free}\setminus \bar{X}_{goal},
\end{array}
\end{equation}
and $V(z)=0$ for all $z$ in the closure of $X_\mathrm{goal}$ (denoted $\bar{X}_\mathrm{goal}$), then $V$ is equal to the value function on $X_\mathrm{free}$. Likewise, if the HJB equation admits a classical solution, then it is equal to the value function.

\section{Graph-Search Oriented Approximations}\label{sec:admissible_heuristics}
Many computational methods for solving the kinodynamic motion problem approximate the set of all possible trajectories by a finite directed graph $(\mathcal{V},E)$,  whose vertices are states in the state space, and whose edges correspond to trajectories between two vertices satisfying (\ref{eq:ode}).
Conceptually, the optimal feasible trajectories restricted to the graph are in some sense faithful approximations of optimal feasible trajectories for the original problem. 
The non-negativity of the cost function (\ref{eq:cost}) enables a nonnegative edge-weight to be assigned to each edge corresponding to the cost of the trajectory in relation with that edge.
The approximated problem can then be addressed using shortest path algorithms for graphs.

The value function $\hat{V}:\mathcal{V}\rightarrow \mathbb{R}$ on the weighted graph is analogous to the value function $V$ in the original problem.
For a vertex $x_0$ in the graph, $\hat{V}(x_0)$ is the cost of a shortest path to one of the goal vertices: $\mathcal{V}\cap X_\mathrm{goal}$.  
Since the feasible trajectories represented by the graph are a subset of the feasible trajectories of the problem we have the inequality
\begin{equation}\label{eq:restriction}
V(z)\leq \hat V(z),\qquad \forall z \in \mathcal{V}.
\end{equation} 

\subsection{Admissible Heuristics}
To carry out an informed search and ensure the optimality of the result, many algorithms require an admissible heuristic $H:X_\mathrm{free} \rightarrow \mathbb{R}$. 
A heuristic $H$ for a problem with value function $V$ is \textit{admissible} if, 
\begin{equation}\label{eq:admissible}
H(z)\leq {V}(z),\qquad \forall z \in X_\mathrm{free}.
\end{equation}
In light of (\ref{eq:restriction}), an admissible heuristic for the kinodynamic motion planning problem will also be admissible for the value function of an approximation to the problem.
For the remainder, the set of candidate heuristics will be restricted to differentiable scalar functions on ${X}_{free}$. 

Since the value function is unknown it is difficult check that (\ref{eq:admissible}) is satisfied for a particular heuristic $H$.
This motivates the first contribution of this paper, a sufficient condition for admissibility that can be checked using the problem data.
\begin{lem}[Admissibility]\label{lem:admissible_condition}
A heuristic $H$  is an admissible heuristic if:
		\begin{equation}\tag{AH1}\label{eq:AH1}
		H(z) \leq 0, \qquad \forall z \in  X_\mathrm{goal},
		\end{equation}
		and
	\begin{equation}\tag{AH2}\label{eq:AH2}
	\left\langle \nabla_{z}H(z),f(z,w)\right\rangle +g(z,w) \geq 0,
	\end{equation} 
for all $u\in\Omega$ and all $z\in X_{free}$.
\end{lem}
\begin{proof}
Choose a feasible trajectory $x$ and associated control signal $u$.
By construction $x(T)\in X_\mathrm{goal}$ so $H(x(T))\leq 0$. 
Then
\begin{equation} \arraycolsep=1.4pt\def\arraystretch{1.5}
\begin{array}{rcl}
H(x(0)) & \leq & H(x(0))-H(x(T)) \\
        & = & -\int_{0}^{T} \frac{d}{dt} H(x(t)) \,dt \\
        & = &  -\int_{0}^{T}  \left\langle \nabla H(x(t)),f(x(t),u(t)) \right\rangle \,dt \\
        & \leq &  \int_{0}^{T}  g(x(t),u(t) \,dt \\
        & = & J(x,u).
\end{array}
\end{equation} 
Note that the fourth step of the derivation follows from assuming (\ref{eq:AH2}). 
Since $H(x(0))\leq J(x,u)$ for any  feasible trajectory and related control we conclude that $H$ provides a lower bound on the cost-to-go from any initial condition. 
By definition the value function $V$ is the greatest lower bound. Thus, 
\begin{equation}
H(z)\leq V(z),\qquad \forall z\in X_\mathrm{free}.
\end{equation}
\end{proof}
Inequalities of the form (\ref{eq:AH1}) and (\ref{eq:AH2}) appear frequently in the optimal control literature where $H$ is considered a smooth subsolution to the HJB equation.
Observe that Lemma \ref{lem:admissible_condition} does not require $V$ or $g$ to be continuous, nor does $f$ have to be differentiable, and hence is quite general.

An immediate application of this result is as a sufficient condition for the admissibility of a candidate heuristic.   
Section \ref{sec:Examples_usage} provides three examples demonstrating this technique.

Another concept related to admissibility is \textit{consistency}~\cite{hart1968formal}. 
This is a stronger property which is also verified by Lemma \ref{lem:admissible_condition} if (\ref{eq:AH1}) is replaced by the condition that $H(z)=0$ for all $z\in X_\mathrm{goal}$.
The proof can be found in the Appendix.

\section{Optimization of Admissible Heuristics}\label{sec:inf_lp}
The second contribution of this paper is a general procedure for computing and optimizing an admissible heuristic.
To motivate the proposed optimization we review duality results developed by Fleming \cite{fleming1988generalized} and refined by Vinter \cite{vinter1993convex}. Stated informally\footnote{This result requires a relaxed notion of a trajectory and some mild technical assumptions on the problem data; cf \cite{vinter1993convex} for details.}, the result applied to our problem is as follows:
\begin{theorem*}[2.1-\cite{vinter1993convex}] 
	Consider the kinodynamic planning problem 
\begin{equation}\arraycolsep=1.4pt\def\arraystretch{1.5}
\begin{array}{rll}
	\underset{x,u}{\min} & \int_0^T g(x(t),u(t))\,\mu(dt) & \\
	{\rm subject\,to:}  & x(0) = x_0 \qquad \forall z \in  X_{goal}, & \\
	& x(T)\in X_{goal}, & \\
	& x(t) \in X_{free} \quad & \forall t \in [0,T], \\
	& \dot{x}(t)=f(x(t),u(t))\quad & {\rm a.e.} \,t\in [0,T], \\
	& u(t) \in \Omega \quad & {\rm a.e.} \, t \in [0,T].
\end{array} \label{eq:primal}\tag{P}
\end{equation}
The dual problem is
\begin{equation}\arraycolsep=1.4pt\def\arraystretch{1.5}
\begin{array}{rll}
\underset{H}{\max} & H(x_0) & \\
{\rm subject\,to:}  & 		H(z) \leq 0 \qquad \forall z \in  X_{goal},
 & \\
& \left\langle \nabla_{z}H(z),f(z,w)\right\rangle +g(z,w) \geq 0 \\
& \forall z\in X_{free},\,{\rm and}\,\, \forall w \in \Omega,
\end{array} \label{eq:dual}\tag{D}
\end{equation}
and strong duality holds. That is, the  optimal values of the two problems coincide.
\end{theorem*}
Since there is no duality gap the optimal value at the initial condition $V(x_0)$ can be obtained by solving the dual problem. 
Observe that the objective of the dual problem is linear and the constraints are affine making it a linear program.
Problem (\ref{eq:dual}) will not yield a particularly good heuristic since it optimizes the heuristic as a single point. However, it does suggest a related optimization to obtain the value function over any subset of $X_{free}$. 
Instead of optimizing $H$ at a single point, we can take the integral with respect to any positive measure $m$ on $X_{free}$.
If the value function is bounded on the support of $m$, then maximizing the integral is equivalent to solving (\ref{eq:dual}) at  almost every point in the support of $m$.
The integral objective is still linear and the problem remains a (infinite dimensional) linear program: 
\begin{equation}\arraycolsep=1.4pt\def\arraystretch{1.5}
\boxed{
\begin{array}{rll}
		\underset{H}{\max} & \int_{X_{free}}H(z)\,\,m(dz) & \\
		{\rm subject\,to:}  & 		H(z) \leq 0 \qquad \forall z \in  X_{goal},
		& \\
		& \left\langle \nabla_{z}H(z),f(z,w)\right\rangle +g(z,w) \geq 0 \\
		& \forall z\in X_{free},\,{\rm and}\,\, \forall w \in \Omega,
		
	\end{array} \label{eq:convex_opt}\tag{LP}}
\end{equation}
Note that (\ref{eq:convex_opt}) reduces to (\ref{eq:dual}) if a discrete measure concentrated at  ${x_0}$ is used.  

To further justify using the objective in (\ref{eq:convex_opt}) to optimize our heuristic we show that the value function is the solution when it is differentiable.
\begin{lem}
	If the value function $V$ is differentiable on $X_{free}\setminus \bar{X}_{goal}$, then it solves (\ref{eq:convex_opt}). 
\end{lem}
\begin{proof}
	(Feasibility) From (\ref{eq:-4}), $V(z)=0$ for all $z \in \bar{X}_{goal}$ so the constraints (\ref{eq:AH1}) and (\ref{eq:AH2}) are satisfied on $\bar{X}_{goal}$.
	Since $V$ is differentiable, it solves the HJB equation (\ref{eq:HJB}). Thus,
	\begin{equation}
	\begin{array}{c}
	\underset{w \in \Omega}{\inf} \left\{ \left\langle \nabla V(z), f(z,w) \right\rangle + g(z,w) \right\}=0, \\ \forall z\in X_\mathrm{free}\setminus \bar{X}_\mathrm{goal}.
	\end{array}
	\end{equation}
	This implies
	\begin{equation}\arraycolsep=1.4pt\def\arraystretch{1.5}
	\begin{array}{c}
	 \left\langle \nabla V(z), f(z,w) \right\rangle + g(z,w) \geq0, \\ \forall z\in X_\mathrm{free}\setminus \bar{X}_\mathrm{goal},\,{\rm and}\, u\in \Omega. 
	 \end{array}
	\end{equation}
	Therefore, (\ref{eq:AH2}) is satisfied. 
	(Optimality) By Lemma \ref{lem:admissible_condition}, a feasible solution $H$ satisfies $H(z)\leq V(z)$ for all $z \in X_{free}$.
	 Thus,
	\begin{equation}
	\int_{X_{free}} H(z)\,\,m(dz) \leq \int_{X_{free}} V(z)\,\,m(dz).  
	\end{equation}
	That is, the value function provides an upper bound on the objective in (\ref{eq:convex_opt}).
	$V$ is a feasible solution so this upper bound is attained and $V$ is therefore an optimal solution.
	%
	
\end{proof}

\section{Sum-of-Squares (SOS) Relaxation to (\ref{eq:convex_opt}) }\label{sec:Characterising-Admissibility-as}
To tackle (\ref{eq:convex_opt}) with standard mathematical programming techniques, we must approximate the set of candidate heuristics by a finite-dimensional subspace.
The proposed basis for this subspace is a finite collection of polynomials.  
The relaxation can then be addressed efficiently using SOS programming.

SOS programming \cite{parrilo2004sum} is a method of optimizing a functional of a polynomial subject to semi-algebraic constraints. 
The technique involves relaxing the semi-algebraic constraints to a sum-of-squares constraint which is equivalent to a semi-definite program (SDP). 
The advantages of this approach are that the approximate solution is guaranteed to be an admissible heuristic, and the relaxation is a convex program which can be solved in polynomial time using interior-point methods.
\subsection{Sum-of-Squares Polynomials}
A polynomial $p\in\mathbb{R}[x]$ in $n$ variables is said to be a sum-of-squares if it can be written as 
\begin{equation}
p(x)=\sum_{k=1}^{d}q_{k}(x)^{2},\label{eq:-5}
\end{equation}
for polynomials $q_{k}(x)$. 
Clearly, $p(x)\geq0$ for all  $x \in \mathbb R^n$. 
Note also that $p(x)$ is a sum-of-squares if and only if it can be written as 
\begin{equation}
p(x)=m(x)^{T}Qm(x),\label{eq:-4}
\end{equation}
for a positive semidefinite matrix $Q$ and the vector of  \textit{\tiny{}
$\left(\begin{array}{c}
n+d\\
n
\end{array}\right)$} monomials $m(x)$ up to degree $d$.
For a polynomial $p$ admitting a decomposition of the form (\ref{eq:-4}) we write $p \in SOS$. 
Equation (\ref{eq:-4}) is a collection of linear equality constraints between the entries of $Q$ and the coefficients of $p(x)$. 
Finding entries of $Q$ such that $Q\succeq0$ and the equality constraints are satisfied is then a semi-definite program (SDP). 
The complexity of finding a solution to this problem using interior-point methods is generally polynomial in the size of $Q$.
This method of analyzing polynomial inequalities has had a profound impact in many fields. As a result there are a number of optimized solvers~\cite{sedumi,sdpt3} and modeling tools~\cite{sostools,yalmip} available.

\subsection{Optimizing the Heuristic}

To proceed with computing a heuristic using the SOS programming framework the problem data must consist of polynomials and intersections of semi-algebraic sets. 
Let
\begin{equation}\arraycolsep=1.4pt\def\arraystretch{1.5}
\begin{array}{rcl}
X_{free} & = & \left\{ z \in \mathbb{R}^n:\,h_z(z)\geq 0 \right\}, \\
\Omega   & = & \left\{ w \in \mathbb{R}^m:\,h_w(w) \geq 0 \right\}, 	
\end{array}
\end{equation}
for polynomials $h_x$ and $h_u$.
Assume also that $f$, $g$ and the candidate heuristic $H$ are polynomials.
Then the admissibility condition 
\begin{equation}\arraycolsep=1.4pt\def\arraystretch{1.5}
\begin{array}{c}
\left\langle \nabla_{x}H(x),f(x,u)\right\rangle +g(x,u)\geq 0,\\
\forall w\in \Omega,\,\,{\rm and}\,\, z\in X_{free},
\end{array}
\label{eq:-6}
\end{equation}
is a polynomial inequality. 
To restrict nonnegativity of the heuristic to $X_{free}$ and $\Omega$, add the auxiliary SOS polynomials $\lambda_{x}(x)\geq0$ and $\lambda_{u}(u)\geq0$ to the equation as 
\begin{equation}\arraycolsep=1.4pt\def\arraystretch{1.5}
\begin{array}{l}
\left\langle \nabla_{x}H(x),f(x,u)\right\rangle +g(x,u) \\
-\lambda_{x}(x)^Th_{x}(x)-\lambda_{u}(u)^Th_{u}(u) \geq 0, \\
\forall w\in \mathbb{R}^m,\,\,{\rm and}\,\, z\in \mathbb{R}^n,
\end{array}\label{eq:-7}
\end{equation}
which trivially implies the positivity of (\ref{eq:-6}) over $X_\mathrm{free}$
and $\Omega$. 
When $H$ is a polynomial, the objective in (\ref{eq:convex_opt}) is linear in the coefficients of $H$.
Thus, it is an appropriate objective for an SOS program.
The SOS program which is solved to obtain an admissible heuristic is then 
\begin{equation}\label{eq:sosformula}\arraycolsep=1.4pt\def\arraystretch{1.5}
\begin{array}{rll}
 \underset{H,\lambda_{x},\lambda_{u}}{\max}& \int_{X_{free}}H(z)\,m(dz)& \\  
 {\rm subject\, to:}& H(x)= 0, \quad \forall x\in \bar{X}_\mathrm{goal},&\\
 &\left\langle \nabla_{x}H(x),f(x,u)\right\rangle +g(x,u)\\
 & -\lambda_{x}(x)^T h_{x}(x)-\lambda_{u}(u)^T h_{u}(u) &\in SOS,\\
& \lambda_{x}(x),\lambda_{u}(u) \in  SOS.&  \\
\end{array} 
\end{equation}

\section{Examples}\label{sec:Examples_usage}
The remainder of the paper is devoted to examples demonstrating how to apply Lemma \ref{lem:admissible_condition} to verify admissibility, and the SOS relaxation of (\ref{eq:convex_opt}). 

\subsection{Verifying Candidate Heuristics}
%
%
%
%

%
The next three examples demonstrate some techniques utilizing Lemma \ref{lem:admissible_condition} to verify the admissibility of a heuristic.
In particular, the Cauchy-Schwarz inequality and the inequality 
\begin{equation}\label{eq:basic_sos}
|2ab|\leq a^2+b^2,
\end{equation}
are often useful. 
In the first example we show how to use Lemma \ref{lem:admissible_condition} to verify a classic heuristic used in kinematic shortest path problems.
\begin{ex}\label{ex:spp}
	Consider a reformulation of the shortest path problem,
	\begin{equation}
	\dot{x}=u, 
	\end{equation}
	where $x\in \mathbb{R}^n$, and $u\in \{ w\in \mathbb{R}^n :\, \Vert w \Vert = 1 \}$.
	The cost which reflects a shortest path objective is 
	\begin{equation}
	J(x,u)=\int_0^T 1\,\mu(dt).
	\end{equation}
	Let the goal set be $\{0\}$.
	We would like to verify the classic heuristic
	\begin{equation}
	H(x)=\Vert x \Vert.
	\end{equation}
	Applying the admissibility Lemma we obtain
	\begin{equation}  \arraycolsep=1.4pt\def\arraystretch{1.5}
	\begin{array}{rcl}
	\left\langle \nabla H(x),f(x,u)\right\rangle +g(x,u) & = & \dfrac{\left\langle x,u\right\rangle }{\Vert x\Vert}+1\\
	& \geq & \dfrac{-\Vert x\Vert\Vert u\Vert}{\Vert x\Vert}+1\\
	& \geq & -1+1\\
	& = & 0,
	\end{array}
	\end{equation}
	which reverifies the fact that the Euclidean distance is an admissible heuristic for the shortest path problem.
	The crux of this derivation is simply applying the Cauchy-Schwarz inequality in the first step.
	
\end{ex}

In the next example, we derive heuristics for two variations of a classic wheeled robot model.

\begin{figure*}\label{fig:dubins_demo}
	\centering{}
	\includegraphics[width=1.0\columnwidth]{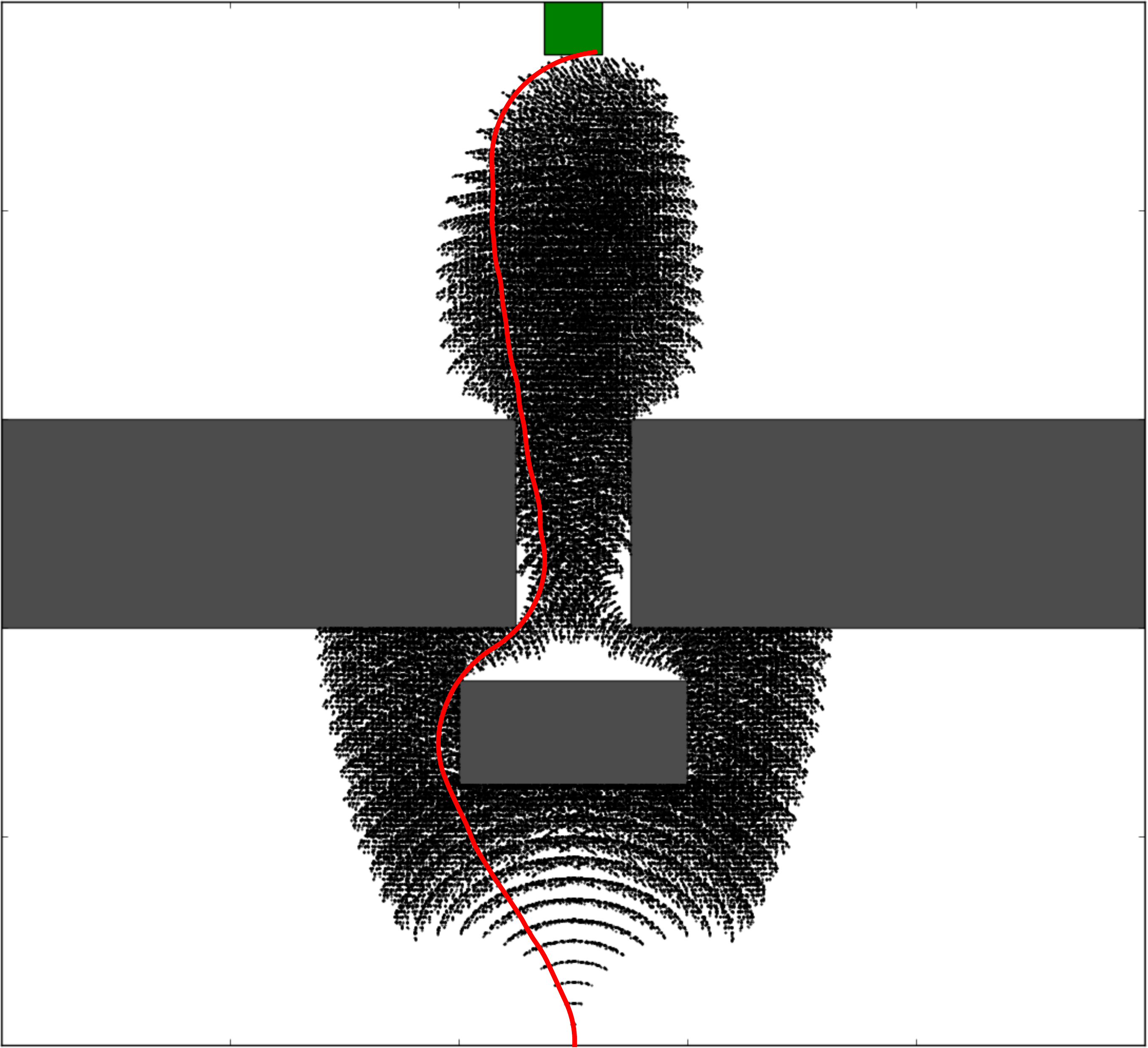}
	\includegraphics[width=1.0\columnwidth]{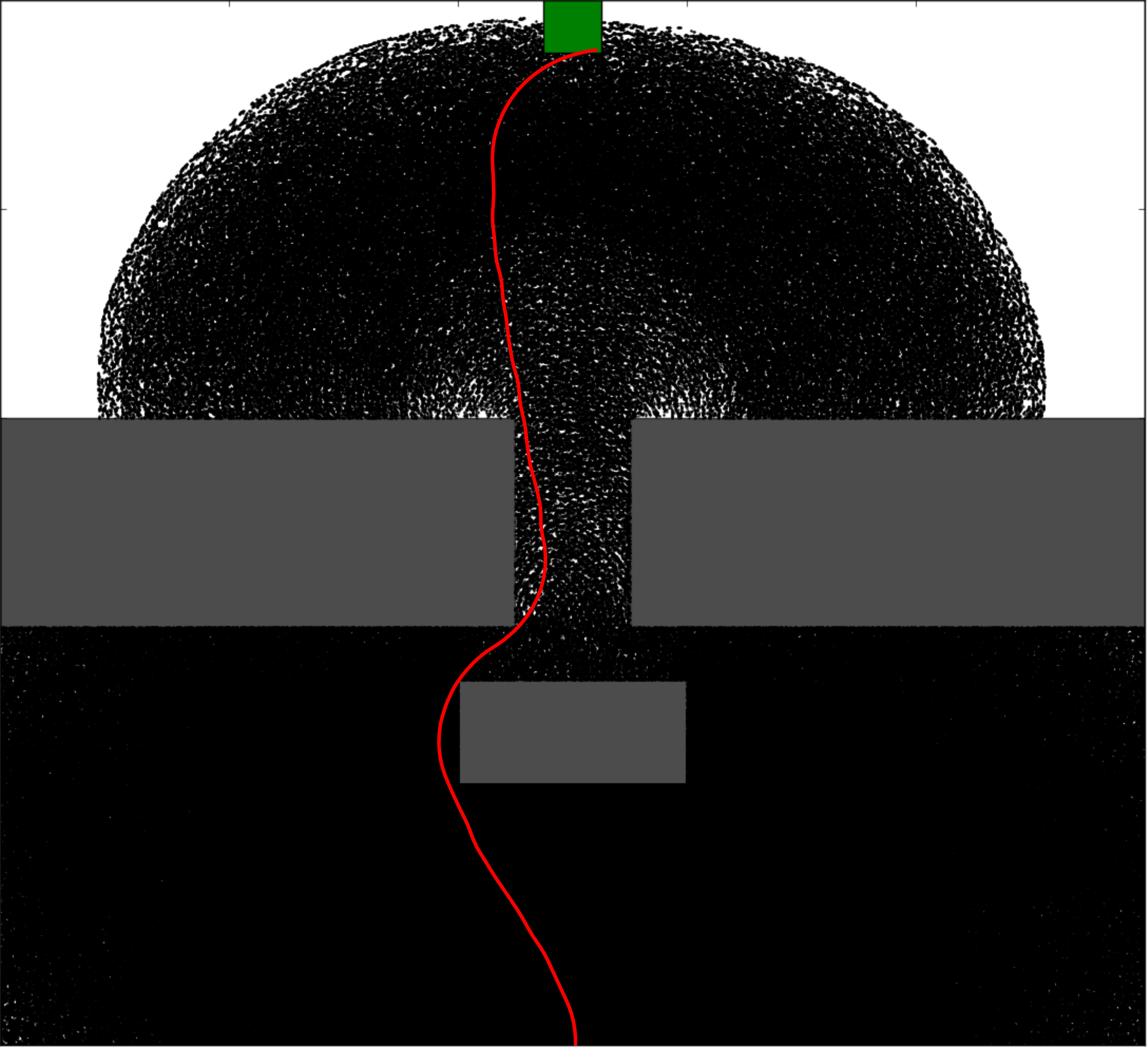}
	\caption{Approximate shortest path in a 2D environment for a simple wheeled robot. The goal set includes a terminal heading specification which explains the right turn at the end of the path. Paths were computed by the GLC method with dots representing the projection of vertices evaluated during the search onto the $x$-$y$ plane. The algorithm was executed with (left) and without (right) the admissible heuristic described in equation (\ref{eq:dubins_heuristic}) of Example \ref{ex:dubins}. The informed GLC method obtains the illustrated solution in 209341 iterations while the standard GLC method obtains the solution in 2380952 iterations.}
\end{figure*}
\begin{ex}\label{ex:dubins}
	Consider a simple wheeled robot with states $(x,y,\theta)^T\in \mathbb{R}^3$ and whose mobility is described by
	\begin{equation}\arraycolsep=1.4pt\def\arraystretch{1.5} \label{eq:reed_shepp_car}
	\begin{array}{rcl}
	\dot{x} & = & \cos(\theta), \\
	\dot{y} & = & \sin(\theta), \\
	\dot{\theta} & = & u.
	\end{array}
	\end{equation}
	Let $X_\mathrm{free}=\mathbb{R}^3$, $X_\mathrm{goal}=\{(0,0,0)^T\}$, and  $\Omega=\mathbb{R}$.
	The cost functional measures the duration of the trajectory. 
	\begin{equation}
	J(x,u)=\int_{[0,T]} 1\,\mu(dt).
	\end{equation}
	Equivalently, this is the path length in the $x$-$y$ plane.
	As a candidate heuristic, take the length of the line segment connecting the $x$-$y$ coordinate to the origin.
	\begin{equation}\label{eq:unicycle_heuristic}
	H_1(x,y,\theta)=\Vert (x,y)^T \Vert .
	\end{equation}
	The intuition being that the shortest path in the absence of the differential constraint will be shorter than the shortest path for the constrained system.
	The admissibility condition is verified for this heuristic using the Cauchy-Schwarz inequality. 
	Inserting the expression for the heuristic into  (\ref{eq:AH2}) yields 
	\begin{equation}\arraycolsep=1.4pt\def\arraystretch{1.75}
	\begin{array}{l}
	\left\langle \nabla H(x,y,\theta),f(x,y,\theta)\right\rangle +g(x,y,\theta,u)\\
	=\dfrac{\left\langle (x,y,0)^{T},(\cos(\theta),\sin(\theta),u)^{T}\right\rangle }{\left\Vert (x,y)^{T}\right\Vert }+1\\
	\geq-\dfrac{\left\Vert (x,y)^{T}\right\Vert \left\Vert (\cos(\theta),\sin(\theta))^{T}\right\Vert }{\left\Vert (x,y)^{T}\right\Vert}+1\\
	\geq-1+1\\
	=0.
	\end{array}
	\end{equation}
	Thus, the heuristic is admissible.
	Next, consider a restriction of the control actions to $\Omega =[-1,1]$.
	The original heuristic remains valid since the old problem is a relaxation of the new problem.
	Additionally, we can consider a second heuristic to augment the first,
	\begin{equation}
	H_2(x)=|\theta|.
	\end{equation}
	With the added constraint, this heuristic satisfies Lemma \ref{lem:admissible_condition},
	\begin{equation}  \arraycolsep=1.4pt\def\arraystretch{1.5}
	\begin{array}{rcl}
	\left\langle \nabla H_2(x),f(x,u)\right\rangle +g(x,u) & = & \dfrac{\theta u}{|\theta|}+1\\
	& \geq & -\dfrac{|\theta||u|}{|\theta|}+1\\
	& \geq & -1+1\\
	& = & 0.
	\end{array}
	\end{equation}
	We can then combine these heuristics in the input constrained problem,
	\begin{equation}\label{eq:dubins_heuristic}
	H(x,y,\theta )=\max \{ \left\Vert (x,y) \right\Vert,|\theta |\}
	\end{equation}

\end{ex}

The heuristic in (\ref{eq:dubins_heuristic}) was used to plan a feasible path in a 2D environment illustrated in Figure \ref{fig:dubins_demo}. 
For the demonstration the singleton goal set was approximated by a small cube centered at $(0,0,0)^T$ as required by the motion planning algorithm. The use of the heuristic reduces the number of iterations of the algorithm by $91\%$. 

The last example considers a problem with a quadratic regulator objective instead of a minimum time objective.
\begin{ex}
	Consider a simple pendulum with dynamics
	\begin{equation}
	\begin{array}{rcl}
	\dot{\theta} & = & \omega,\\
	\dot{\omega} & = & \sin(\theta)+u.
	\end{array}
	\end{equation}
	Let $X_\mathrm{free}=\mathbb{R}^{2}$, $\Omega=[-1,1]$, and $X_\mathrm{goal}=\{(0,0)^{T}\}$.
	The cost function will be a typical quadratic regulator cost.
	\begin{equation}
	J=\int_{0}^{T}\rho(\theta(t)^{2}+\omega(t)^{2}+u(t)^{2})\,\,\mu(dt).
	\end{equation}
	Select a heuristic of the form 
	\begin{equation}
	H(\theta,\omega)=\frac{\alpha}{2}\left(\theta^{2}+\omega^{2}\right).
	\end{equation}
	Checking the admissibility condition,
	\begin{equation}\arraycolsep=1.4pt\def\arraystretch{1.5}
	\begin{array}{l}
	\left\langle \nabla H(x),f(x,u)\right\rangle +g(x,u)\\
	=\alpha\theta\omega+\alpha\omega\sin(\theta)+\alpha\omega u+\rho(\theta^{2}+\omega^{2}+u^{2})\\
	\geq-|2\alpha\theta\omega|-|\alpha\omega u|+\rho(\theta^{2}+\omega^{2}+u^{2})\\
	\geq-\alpha(\theta^{2}+\omega^{2})-\frac{1}{2}\alpha(\omega^{2}+u^{2})+\rho(\theta^{2}+\omega^{2}+u^{2})\\
	=(\rho-\frac{3}{2}\alpha)\theta^{2}+(\rho-\frac{3}{2}\alpha)\omega^{2}+(\rho-\frac{1}{2}\alpha)u^{2}.
	\end{array}
	\end{equation}
	The inequality in (\ref{eq:basic_sos}) was used in the third step of this derivation.
	The above quantity is nonnegative and therefore $H$ is an admissible heuristic for $\alpha \leq \frac{2}{3}\rho$.
\end{ex}

\subsection{SOS Heuristic Optimization Examples}
\label{sec:Examples}

The next two examples demonstrate the SOS programming formulation described in Section \ref{sec:Characterising-Admissibility-as}. 
In both examples, a closed form solution for the value function $V(x)$ is known for $X_\mathrm{free}=\mathbb{R}$ and $X_\mathrm{free}=\mathbb{R}^2$ respectively. 
This solution provides a useful point of comparison for the computed heuristics.
These examples also illustrate flexibility in selecting a measure on $X_\mathrm{free}$.
Intuitively, the measure is a tuning parameter that places greater emphasis on the optimization over certain subsets of $X_{free}$.  
These example problems were implemented using the SOS module in YALMIP \cite{yalmip} and solved using SDPT3 for the underlying semidefinite program \cite{sdpt3}. 
To further illustrate the approach, YALMIP scripts for these examples can be found in~\cite{SOS_heuristics}.

\begin{ex}[Single Integrator (1D)] \label{ex:1d_kine}
To illustrate the procedure, we revisit Example \ref{ex:spp} in the 1-dimensional case. 
The differential constraint is given by $f(x,u)=u$ where $x,u\in\mathbb{R},$
$X_\mathrm{free}=[-1,1],$ $\Omega=[-1,1]$, and $X_\mathrm{goal}=\{0\}$. 
Again we use the minimum time objective where $g(x,u)=1$.
The value function $V(x)=|x|$ is obtained by inspection.
The heuristic is parameterized by the coefficients of a univariate polynomial of degree $2d$
\begin{equation}
H(x)=\sum_{i=0}^{2d}c_{i}x^{i}.
\end{equation}
Using a discrete measure on $[-1,1]$ concentrated at the boundary the SOS program is, 
\begin{equation}\arraycolsep=1.4pt\def\arraystretch{1.5}  \label{eq:simple_demo}
\begin{array}{rll}
\underset{H,\lambda_{x},\lambda_{u}}{\max} & \left\{ H(1)+H(-1)\right\} & \\
{\rm subject\, to:} &H(0)=0, & \\
\left(\frac{d}{dx}H(x)\right)u+1 & &\\
-\lambda_{x}(x)(1-x^{2})-\lambda_{u}(u)(1-u^{2}) & \in SOS,&\\
\lambda_{x}(x),\lambda_{u}(u) & \in SOS.&
\end{array}
\end{equation}
The numerical solution for polynomial heuristics with increasing degree is shown in Figure \ref{fig:1dheuristic}. 
%

\begin{figure}
\centering{}\includegraphics[width=1.0 \columnwidth]{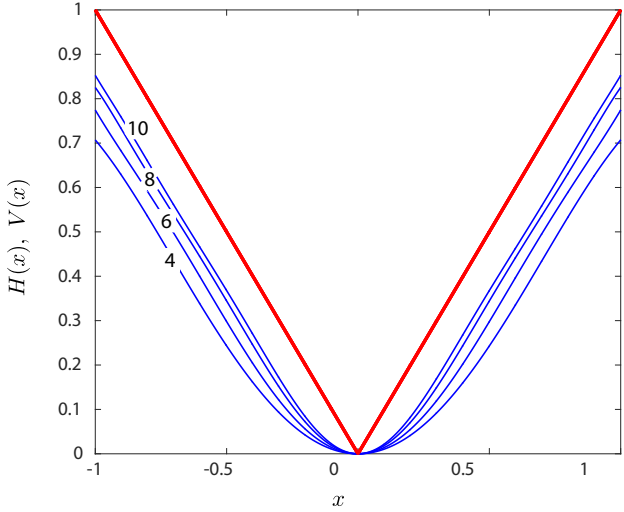}\caption{Univariate polynomial heuristics of degrees 4, 6, 8, and 10 for the 1D single integrator shown in blue. The value function is shown in red. Polynomial heuristics with higher degree provide better underestimates of the value function. \label{fig:1dheuristic}}
\end{figure}
\end{ex}

\begin{figure} \label{fig:1d_dbl}
	\centering
	\includegraphics[width=\columnwidth]{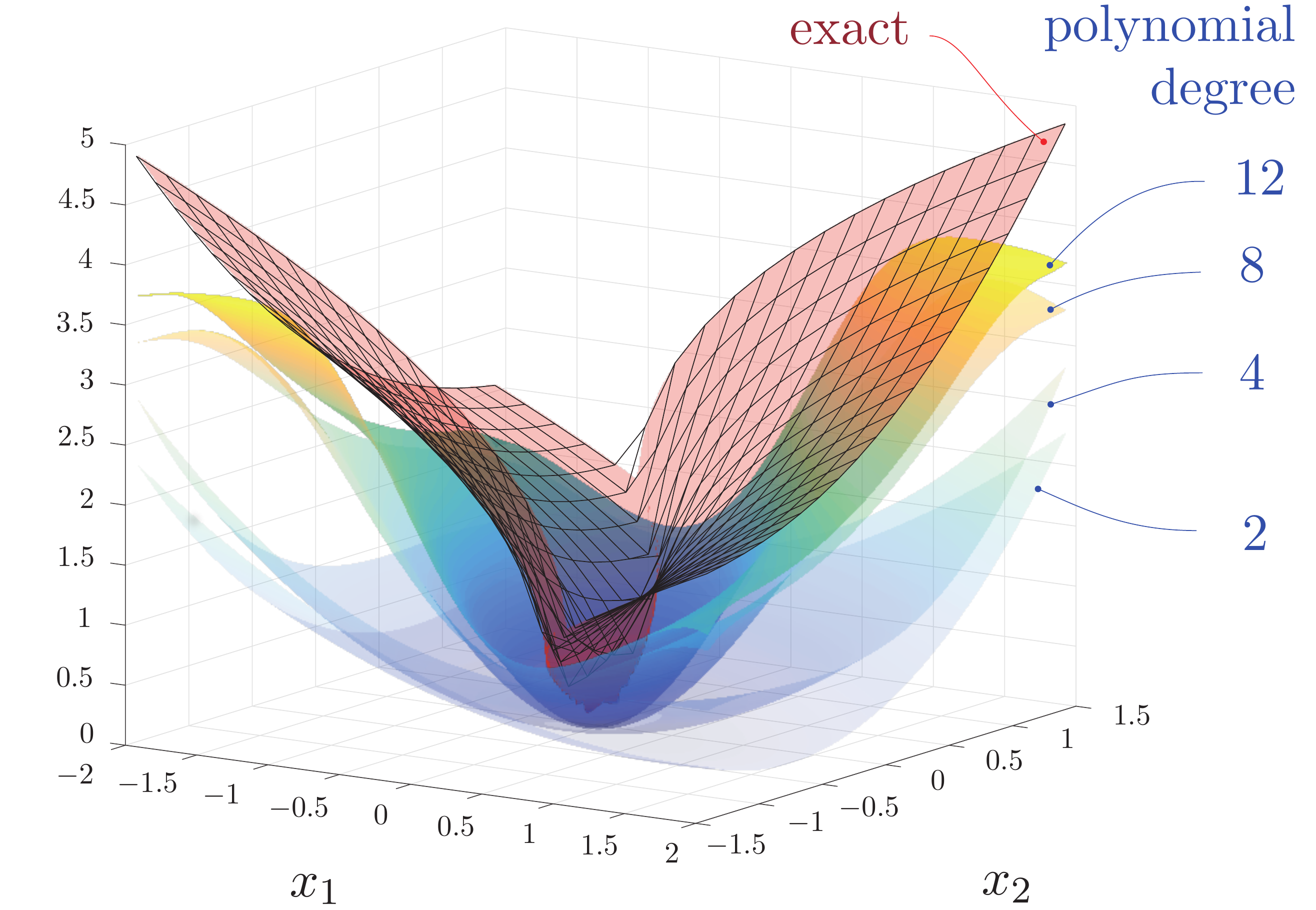}
	\caption{Polynomial heuristics of degree 2, 4, 8, and 12 for the 1D double integrator in comparison with the known value function shown in red. Polynomial heuristics with higher degree provide better underestimates of the value function. }
\end{figure}

\begin{ex}[Double Integrator (1D)]

As an example with a more complex value function take the vector field
\begin{equation}
\left(\begin{array}{c}
\dot{x}_{1}\\
\dot{x}_{2}
\end{array}\right)=\left(\begin{array}{c}
x_{2}\\
u
\end{array}\right),
\end{equation}
and the minimum-time cost functional
\begin{equation}
J(x,u)=\int_0^T 1\,dt.
\end{equation}
The remaining problem data for this example are $X_\mathrm{free}=[-3,3]^2$, $\Omega = [-1,1]$, and $X_\mathrm{goal} = \{(0,0)^T\}$. 
%
%

%
Polynomial heuristics of degree $2d$ of the form
\begin{equation}
H(x_1, x_2) = \sum_{p + q\leq 2d} c_{p, q}\; x_1^p x_2^q,
\end{equation}
are computed for $X_\mathrm{free} = [-3, 3]^2$ and $\Omega = [-1, 1]$.
In this example the support for the measure $m$ is $S=[-2,2] \times [-\sqrt 2, \sqrt 2]$. 
This focuses the optimization in a region around the goal while maintaining admissibility of the heuristic over all of $X_\mathrm{free}$.
The SOS program is formulated as follows
\begin{equation}\arraycolsep=1.4pt\def\arraystretch{1.5} 
\begin{array}{rl}
\underset{H,\lambda_{x},\lambda_{u}}{\max}&\int_{S}H(x_{1},x_{2})\,\, m(dx) \\ 
{\rm subject\, to:} & H(0,0)=0,\\ 
 \left(\nabla H(x_{1},x_{2})\right)\left(\begin{array}{c}
x_{2}\\
u
\end{array}\right)+1 & \\
 -\lambda_{x_{1}}(x_{1})\left(9-x_{1}^{2}\right) &  \\
-\lambda_{x_{2}}(x_{2})\left(9-x_{2}^{2}\right) &  \\
-\lambda_{u}(u)(1-u^{2}) & \in SOS,\\
\lambda_{x_{1}}(x_{1}),\lambda_{x_{2}}(x_{2}),\lambda_{u}(u)  & \in SOS.
\end{array} 
\end{equation}

The optimized heuristics of increasing degree are shown in Figure \ref{fig:1d_dbl} together with the value function for $X_\mathrm{free}=\mathbb{R}^n$. 
\end{ex}

\begin{rem}
	In the last example, the optimization focused on the region $[-2,2] \times [-\sqrt 2, \sqrt 2]$ instead of $[-3,3]^2$.
	The reason for this is that some states in $[-3,3]^2$ cannot reach the goal without leaving $[-3,3]^2$.
	As a consequence the value function is unbounded at these states. 
	A remarkable observation is that the resulting SOS program does not admit a maximum when the integral includes a subset of $X_\mathrm{free}$ where the value function is unbounded.   
	This is entirely consistent with the theoretical results since the heuristic is free to go unbounded over this set as well.
	
\end{rem}

\section{Conclusions and Future Work}
We have provided a sufficient condition for verifying the admissibility of a candidate heuristic in general kinodynamic motion planning problems and demonstrated through several examples how to utilize the condition.
The admissibility condition was then used to formulate a linear program over the space of candidate heuristics whose optimal solution coincides with classical solutions to the HJB equation.
Using sum-of-squares programming we were able to provide approximate solutions to this optimization in polynomial time.
This provides the first general synthesis procedure for admissible heuristics to kinodynamic motion planning problems. 

Automatic synthesis of admissible heuristics in kinodynamic motion planning will be a useful asset to many of the recently developed planning algorithms.
Efforts to further develop this technique are being pursued. 
In the sequel, symmetry reduction techniques from optimal control theory will be applied to reduce the size of the resulting sum-of-squares program. 
We will also explore using the DSOS and SDSOS~\cite{ahmadi2014dsos} programming techniques which would enable using polynomial heuristics with higher degree.

\bibliographystyle{ieeetr}
\bibliography{references}

\begin{thebibliography}{10}

\bibitem{hart1968formal}
P.~E. Hart, N.~J. Nilsson, and B.~Raphael, ``A formal basis for the heuristic
  determination of minimum cost paths,'' {\em Systems Science and Cybernetics,
  IEEE Transactions on}, vol.~4, no.~2, pp.~100--107, 1968.

\bibitem{karaman2010optimal}
S.~Karaman and E.~Frazzoli, ``Optimal kinodynamic motion planning using
  incremental sampling-based methods,'' in {\em 49th IEEE conference on
  decision and control (CDC)}, pp.~7681--7687, IEEE, 2010.

\bibitem{hauser2015asymptotically}
K.~Hauser and Y.~Zhou, ``Asymptotically optimal planning by feasible
  kinodynamic planning in state-cost space,'' {\em arXiv preprint
  arXiv:1505.04098}, 2015.

\bibitem{paden2016generalized}
B.~Paden and E.~Frazzoli, ``A generalized label correcting method for optimal
  kinodynamic motion planning,'' {\em arXiv preprint arXiv:1607.06966}, 2016.

\bibitem{informedRRT}
J.~D. Gammell, S.~S. Srinivasa, and T.~D. Barfoot, ``Informed {RRT*}: Optimal
  incremental path planning focused through an admissible ellipsoidal
  heuristic,'' in {\em International Conference on Intelligent Robots and
  Systems}, 2014.

\bibitem{batchInformedRRT}
J.~D. Gammell, S.~S. Srinivasa, and T.~D. Barfoot, ``Batch informed trees
  (bit*): Sampling-based optimal planning via the heuristically guided search
  of implicit random geometric graphs,'' in {\em International Conference on
  Robotics and Automation}, pp.~3067--3074, IEEE, 2015.

\bibitem{fleming1988generalized}
W.~H. Fleming and D.~Vermes, ``Generalized solutions in the optimal control of
  diffusions,'' in {\em Stochastic Differential Systems, Stochastic Control
  Theory and Applications}, pp.~119--127, Springer, 1988.

\bibitem{vinter1993convex}
R.~Vinter, ``Convex duality and nonlinear optimal control,'' {\em SIAM journal
  on control and optimization}, vol.~31, no.~2, pp.~518--538, 1993.

\bibitem{parrilo2004sum}
P.~A. Parrilo, {\em Structured semidefinite programs and semialgebraic geometry
  methods in robustness and optimization}.
\newblock PhD thesis, Citeseer, 2000.

\bibitem{yalmip}
J.~Lofberg, ``{YALMIP}: A toolbox for modeling and optimization in {MATLAB},''
  in {\em Computer Aided Control Systems Design, 2004 IEEE International
  Symposium on}, pp.~284--289, IEEE, 2004.

\bibitem{SOS_heuristics}
B.~Paden, V.~Varricchio, and E.~Frazzoli, ``Sum-of-squares heuristic synthesis
  for kinodynamic motion planning,''
\newblock Available at:
  \url{https://github.com/bapaden/Sum_of_Squares_Admissible_Heuristics/releases}.

\bibitem{crandall1983viscosity}
M.~G. Crandall and P.-L. Lions, ``Viscosity solutions of hamilton-jacobi
  equations,'' {\em Transactions of the American Mathematical Society},
  vol.~277, no.~1, pp.~1--42, 1983.

\bibitem{sedumi}
J.~F. Sturm, ``Using {SeDuMi} 1.02, a {MATLAB} toolbox for optimization over
  symmetric cones,'' {\em Optimization methods and software}, vol.~11, no.~1-4,
  pp.~625--653, 1999.

\bibitem{sdpt3}
K.-C. Toh, M.~J. Todd, and R.~H. Tutuncu, ``{SDPT3}-a {MATLAB} software package
  for semidefinite programming, version 1.3,'' {\em Optimization methods and
  software}, vol.~11, no.~1-4, pp.~545--581, 1999.

\bibitem{sostools}
S.~Prajna, A.~Papachristodoulou, and P.~A. Parrilo, ``Introducing {SOSTOOLS}: A
  general purpose sum-of-squares programming solver,'' in {\em 41st Conference
  on Decision and Control}, vol.~1, pp.~741--746, IEEE, 2002.

\bibitem{ahmadi2014dsos}
A.~A. Ahmadi and A.~Majumdar, ``{DSOS} and {SDSOS} optimization: {LP} and
  {SOCP}-based alternatives to sum-of-squares optimization,'' in {\em 48th
  Annual Conference on Information Sciences and Systems}, pp.~1--5, IEEE, 2014.

\end{thebibliography}

\section*{Appendix}
Consistency of a heuristic is a type of triangle inequality. 
To define consistency, the value function and heuristic for the kinodynamic motion planning problem must be parametrized by the goal set. 
This is denoted $V(z;X_\mathrm{goal})$ and $H(z;X_\mathrm{goal})$.
A heuristic $H(\,\cdot\,;X_\mathrm{goal})$ is \textit{consistent} if,
\begin{equation}\arraycolsep=1.4pt\def\arraystretch{1.5}
\begin{array}{l}
H(z;X_\mathrm{goal})=0,\quad \forall z\in X_\mathrm{goal},\\
H(z;X_\mathrm{goal}) \leq V(z;\{y\}) + H(y;X_\mathrm{goal}), \quad \forall y,z\in X_\mathrm{free}.
\end{array}
\end{equation}
Note that the inequality above involves the optimal cost-to-go from $z$ to $y$. 
\begin{lem}[Consistency]\label{lem:consistency_condition}
	A heuristic $H(\,\cdot \,;X_\mathrm{goal})$  is consistent if:
	\begin{equation}\tag{CH1}\label{eq:CH1}
	H(z;X_\mathrm{goal}) = 0, \qquad \forall z \in  X_\mathrm{goal},
	\end{equation}
	and
	\begin{equation}\tag{CH2}\label{eq:CH2}
	\left\langle \nabla_{z}H(z;X_\mathrm{goal}),f(z,w)\right\rangle +g(z,w) \geq 0,
	\end{equation} 
	for all $u\in\Omega$ and all $z\in X_{free}$.
\end{lem}

The proof is nearly identical to that of Lemma \ref{lem:admissible_condition}.
\begin{proof}
	Choose a trajectory $x$ and associated control signal $u$ such that $x(0)=z$ and $x(T)=y$.
	Then
	\begin{equation} \arraycolsep=1.4pt\def\arraystretch{1.5}
	\begin{array}{ll}
	H(x(0);X_\mathrm{goal})-H(x(T);X_\mathrm{goal})& \\
	=-\int_{0}^{T} \frac{d}{dt} H(x(t);X_\mathrm{goal}) \,\mu(dt) &\\
	=  -\int_{0}^{T}  \left\langle \nabla H(x(t);X_\mathrm{goal}),f(x(t),u(t)) \right\rangle \,\mu(dt)& \\
	\leq  \int_{0}^{T}  g(x(t),u(t) \,\mu(dt) & \\
	= J(x,u).&
	\end{array}
	\end{equation} 
	Thus, $H(z;X_\mathrm{goal})-H(y;X_\mathrm{goal})$ lower bounds $J(x,u)$ for any trajectory starting at $z$ and terminating at $y$.
	Since $V(\,\cdot \,;{y})$ is the greatest lower bound to the cost of such trajectories we have 
	\begin{equation}
	H(z;X_\mathrm{goal})-H(y;X_\mathrm{goal}) \leq V(z;{y}),\qquad \forall y,z\in X_\mathrm{free}.
	\end{equation}
	Rearranging the expression above yields the definition of consistency for $H(\,\cdot\,;X_\mathrm{goal})$.
\end{proof}

\end{document}